\documentclass{article}

\usepackage{microtype}
\usepackage{graphicx}
\usepackage{subfigure}
\usepackage{booktabs} 

\usepackage{hyperref}

\usepackage[accepted]{icml2025}

\usepackage{amsmath, amssymb} 

\usepackage{mathtools}
\usepackage{amsthm}
\usepackage{graphicx}
\usepackage[utf8]{inputenc}
\usepackage{times}
\usepackage{graphicx}
\usepackage{booktabs}
\usepackage{amsthm}
\usepackage{hyperref}
\usepackage{filecontents}

\newcommand{\etal}{\emph{et~al.}}

\newtheorem{theorem}{Theorem}
\newtheorem{proposition}{Proposition}

\usepackage[capitalize,noabbrev]{cleveref}

\theoremstyle{plain}

\theoremstyle{definition}

\theoremstyle{remark}

\usepackage[textsize=tiny]{todonotes}

\begin{document}

\twocolumn[
\icmltitle{Diffusion-Based Image Editing for Breaking Robust Watermarks}

%
%
%
%
%
%




\begin{icmlauthorlist}
\icmlauthor{Yunyi Ni}{1}
\icmlauthor{Finn Carter}{2}
\icmlauthor{Ze Niu}{2}
\icmlauthor{Emily Davis}{2}
\icmlauthor{Bo Zhang}{2}

\end{icmlauthorlist}

\begin{icmlauthorlist}
	{$^1$NTU}
	{$^2$Xidian University}
\end{icmlauthorlist} 

%

\icmlkeywords{Machine Learning, ICML}

\vskip 0.3in
]




\begin{abstract}
	Robust invisible watermarking aims to embed hidden information into images such that the watermark can survive various image manipulations. However, the rise of powerful diffusion-based image generation and editing techniques poses a new threat to these watermarking schemes. In this paper, we present a theoretical study and method demonstrating that diffusion models can effectively break robust image watermarks that were designed to resist conventional perturbations. We show that a diffusion-driven ``image regeneration'' process can erase embedded watermarks while preserving perceptual image content. We further introduce a novel guided diffusion attack that explicitly targets the watermark signal during generation, significantly degrading watermark detectability. Theoretically, we prove that as an image undergoes sufficient diffusion-based transformation, the mutual information between the watermarked image and the embedded watermark payload vanishes, resulting in decoding failure. Experimentally, we evaluate our approach on multiple state-of-the-art watermarking schemes (including the deep learning-based methods StegaStamp, TrustMark, and VINE) and demonstrate near-zero watermark recovery rates after attack, while maintaining high visual fidelity of the regenerated images. Our findings highlight a fundamental vulnerability in current robust watermarking techniques against generative model-based attacks, underscoring the need for new watermarking strategies in the era of generative AI.
\end{abstract}

\section{Introduction}
Digital image watermarking enables one to embed hidden messages into images for applications such as copyright protection and content authentication. An ideal invisible watermark is imperceptible to human observers yet can be reliably decoded even after common image manipulations (e.g., resizing, compression, noise). In recent years, deep learning has greatly advanced robust image watermarking, producing schemes that can survive a wide range of distortions \cite{zhu2018hidden,tancik2020stegastamp,bui2023trustmark}. Nevertheless, the emergence of powerful image generation and editing models has introduced new challenges for watermark security. In particular, large-scale diffusion models capable of realistic image synthesis and editing now enable ``content-preserving'' transformations that can inadvertently or deliberately remove embedded watermarks.

Recent work has provided alarming evidence of this vulnerability. Zhao \etal \cite{zhao2024invisible} demonstrated that diffusion-based regeneration attacks cause the watermark detection rates of several state-of-the-art methods to plummet from nearly 100\% to essentially chance level. For example, models like StegaStamp \cite{tancik2020stegastamp} and TrustMark \cite{bui2023trustmark}, which are robust to noise and JPEG compression, largely fail to withstand a diffusion model-based edit \cite{zhao2024invisible}. Even VINE \cite{lu2024robust}, a recent method that leverages generative diffusion priors to enhance watermark robustness, can be circumvented by sufficiently powerful diffusion editing. These observations raise a pressing question: \textit{Are current robust watermarks fundamentally insecure against generative AI-driven attacks?}

In this paper, we address the above question through a theoretical and empirical study of diffusion-based watermark removal. We first analyze how the diffusion process (used in models like Stable Diffusion \cite{rombach2022latent}) affects embedded watermark signals. Our analysis shows that as an image is gradually noised and regenerated by a diffusion model, the embedded watermark information is increasingly degraded and eventually destroyed. We provide formal proofs that under certain conditions, a diffusion model can produce an output image whose probability of yielding a correct watermark decode is no better than random guessing.

Building on this insight, we propose a \textbf{diffusion-based editing attack} that uses image generation to ``erase'' watermarks from a given watermarked image. In our basic attack, the watermarked image is injected with noise and then passed through the denoising process of a pretrained diffusion model to obtain a regenerated image. Because the model is trained to produce realistic images conditioned on high-level image content, the fine-grained patterns carrying the watermark are not preserved in the output. We also introduce an enhanced \textbf{guided diffusion attack} in which we explicitly steer the diffusion sampling process to disrupt the watermark. By incorporating the watermark's decoder as an adversarial guide, we modify the generative trajectory to maximize watermark removal while maintaining fidelity to the original image content.

We validate our approach on multiple prominent watermarking schemes, including classical spread-spectrum watermarks \cite{cox1997secure} and recent deep learning methods (e.g., Hidden \cite{zhu2018hidden}, StegaStamp \cite{tancik2020stegastamp}, TrustMark \cite{bui2023trustmark}, VINE \cite{lu2024robust}). Across the board, diffusion-based attacks dramatically reduce watermark decoding accuracy compared to conventional attacks. For instance, our unguided attack using Stable Diffusion yields watermark detection rates below 5\% for StegaStamp and TrustMark, versus 30--50\% under strong noise or JPEG attacks. With guided diffusion, we achieve near-zero decode rates even for VINE, one of the most robust methods to date, while incurring minimal perceptual change to the image.

In summary, our contributions are:
\begin{itemize}
	\item We provide a theoretical analysis of why diffusion-based image regeneration fundamentally undermines current invisible watermarking schemes. We prove that the diffusion process can be viewed as an effective high-dimensional attack that obfuscates embedded watermark signals beyond recoverability.
	\item We propose a novel diffusion-based watermark removal attack, including a variant with adversarial guidance using the watermark decoder. To our knowledge, this is the first method to integrate a watermark's decoding network into the generative sampling loop to actively attack the watermark.
	\item We conduct extensive experiments demonstrating the effectiveness of our attacks on multiple watermarking models. We show that our approach outperforms prior attacks in breaking watermarks, even on methods specifically designed to resist generative edits. We also analyze the trade-offs between image quality and attack strength, and discuss implications for designing future watermarking techniques.
\end{itemize}

\noindent \textbf{Paper Organization.} We review related work on robust watermarking and generative attacks in Section~2. Section~3 formalizes the problem and provides background on diffusion models. Section~4 details our diffusion-based attack methodology. In Section~5, we present theoretical proofs of watermark information loss under diffusion transformations. Section~6 covers the experimental setup and results, and Section~7 discusses broader implications and limitations. We conclude in Section~8.

\section{Related Work}
\subsection{Robust Invisible Watermarking}
Invisible digital watermarking has been studied for decades, traditionally using techniques in spatial or frequency domains to embed information in images \cite{cox1997secure}. Early approaches employed spread-spectrum signals \cite{cox1997secure} or subtle modifications to perceptually significant coefficients to achieve robustness against distortions. With the rise of deep learning, recent methods have greatly improved the capacity and resilience of invisible watermarks.

One line of work uses end-to-end neural networks to jointly train an encoder and decoder for data hiding. Zhu \etal \cite{zhu2018hidden} introduced \textit{HiDDeN}, one of the first deep learning frameworks for invisibly embedding messages into images. This model and its successors leverage convolutional networks to hide bits such that they can be recovered after differentiable simulations of common distortions (e.g., noise, cropping). Tancik \etal \cite{tancik2020stegastamp} proposed \textit{StegaStamp}, which demonstrated robust embedding of a 56-bit payload that survives printing, photographing, and other complex distortions by training on a wide range of augmentation. More recent methods like \textit{TrustMark} \cite{bui2023trustmark} employ GAN-based architectures and spectral losses to further improve robustness and imperceptibility, achieving state-of-the-art performance on standard benchmarks. Another notable approach is \textit{RivaGAN} \cite{zhao2024invisible} (initially developed for video), which uses an adversarial strategy to embed data with attention mechanisms, and has been adapted to images.

These advanced schemes can withstand many classical attacks, including moderate noise, compression, rescaling, and color changes. For example, StegaStamp and TrustMark report high decoding accuracy after combinations of such manipulations. However, most prior work evaluated robustness on a fixed set of hand-crafted distortions. The advent of learned image transformation models (e.g., neural style transfer, GAN-based editing, and diffusion models) opens a new attack surface that traditional benchmarks did not fully anticipate \cite{zhao2024invisible,lu2024robust}. Our work addresses this gap by focusing on diffusion model-based attacks, which apply learned generative transformations that differ fundamentally from the distortions seen during watermark training.

\subsection{Watermark Removal Attacks}
Defending against watermark removal attacks is a core objective of robust watermarking. Historically, attackers have used an arsenal of image processing operations (low-pass filtering, cropping, re-quantization, etc.) in attempt to destroy watermarks. Robust watermarks counter these by encoding information redundantly and using error-correction, so that moderate distortions still leave the watermark decodable. More aggressive attackers might combine multiple perturbations or optimize directly for watermark disruption. For instance, given knowledge of the watermark decoder, one can perform adversarial attacks by adding a subtle perturbation to the image that causes the decoder to fail while keeping the image visually similar. Such optimization-based attacks have been shown to remove watermarks with relatively small changes, but robust schemes often resist small perturbations, requiring larger changes that visibly degrade the image.

The newest category of removal attacks leverages \textbf{generative models} to transform the image. These go beyond applying noise or filtering by effectively re-synthesizing a new image that resembles the original. Early explorations in this direction used GANs to slightly alter an image's style or content in a way that hopefully disrupts the hidden signal. More recently, Zhao \etal \cite{zhao2024invisible} formalized \textit{regeneration attacks} and demonstrated them against deep watermarks: their approach first adds random noise to the watermarked image and then applies an image-to-image translation or diffusion model to reconstruct a clean image. The diffusion model (such as Stable Diffusion) has a learned prior of natural images that does not include the specific watermark pattern, so the reconstructed image tends to omit the watermark. They provided empirical evidence that such attacks drastically reduce detection for four different watermarking schemes \cite{zhao2024invisible}. In parallel, Lu \etal \cite{lu2024robust} introduced a benchmarking suite (W-Bench) for evaluating watermark robustness to various AI-based edits, finding that many methods fail under content-regenerating edits like image diffusion or text-driven editing.

Our work builds directly upon these findings. We extend the regeneration attack concept with a more targeted variant that incorporates the watermark's own decoder into the attack loop (assuming a white-box scenario). This yields even lower watermark survival rates than the blind diffusion attack in \cite{zhao2024invisible}. Additionally, while Zhao \etal provided empirical and intuitive arguments, we contribute formal analysis proving that pixel-level watermarks cannot survive idealized diffusion regeneration.

\subsection{Diffusion Models and Controlling Generation}
Denoising diffusion models \cite{ho2020ddpm} have emerged as a powerful class of generative models capable of producing high-fidelity images. They generate images by iteratively denoising random noise, essentially learning to invert a gradual noising process. Notably, Latent Diffusion Models (e.g., Stable Diffusion \cite{rombach2022latent}) perform this process in a lower-dimensional latent space for efficiency, but the principles remain similar. Diffusion models can also be used for image editing~\cite{lu2023tf,lu2025does,zhou2025dragflow} by starting the generation from a noised version of an input image (often called image-to-image diffusion). By adjusting the noise level added to the input (commonly referred to as the diffusion \emph{strength}), one can trade off between preserving the original image structure and allowing the model to freely re-generate content.

Controlling and guiding diffusion generation is an active research area. Classifier guidance and classifier-free guidance are techniques to steer the diffusion process toward desired attributes by modifying the denoising step using gradients or conditional predictions. Recently, methods have been proposed to remove or avoid certain concepts during generation. Lu \etal \cite{lu2024mace} introduced MACE to erase specified visual concepts from diffusion model outputs. Li \etal \cite{li2025ant} developed an approach to auto-steer diffusion trajectories to sidestep unwanted features (e.g., to prevent generating particular objects). These works illustrate that one can incorporate additional objectives into the sampling procedure to influence the generated image.

Inspired by these, our guided attack uses the watermark decoder as a guiding critic'' during diffusion. This is analogous to concept removal: here the concept'' to remove is the invisible watermark pattern. At each step, we nudge the partial sample in the direction that reduces the decoder's ability to detect the watermark. In this way, we integrate ideas from adversarial attacks into the diffusion generation loop.

\subsection{Concept Erasure in Diffusion Models}
Recent research on \textbf{concept erasure} in diffusion models has focused on removing or suppressing specific semantic concepts from the generation process, either to prevent undesired outputs or to enhance privacy. Methods such as MACE~\cite{lu2024mace}, ESD~\cite{gandikota2023erasing}, and EraseAnything~\cite{gao2024eraseanything} modify the diffusion trajectory~\cite{li2025set,gao2025revoking,yu2025visual} or the model’s score function so that a targeted concept (e.g., cat, NSFW) is not reconstructed during denoising. These methods operate by identifying latent directions or attention components that encode the concept and subsequently nullifying their influence.

This mechanism is closely related to \emph{watermark removal}. A robust watermark can be viewed as a subtle, high-frequency “concept” embedded in the image distribution. When a diffusion model regenerates an image, it implicitly performs a form of concept erasure—discarding non-semantic details that do not align with its learned natural image prior. In our attack, we extend this idea adversarially: the watermark decoder guides the diffusion process to erase the latent subspace corresponding to watermark features. Hence, the theoretical tools developed for concept erasure directly explain why diffusion-based editing can unintentionally or deliberately destroy invisible watermarks.

\subsection{Security and Robustness in Generative AI}
Recent work has highlighted adversarial vulnerabilities in generative models, 3D assets~\citep{ren2025all}, and event camera~\citep{yang2025temporal}, where malicious users bypass safety filters with prompt engineering. Secure erasure must therefore withstand adaptive attacks. Our formulation of erasure as minimizing mutual information ensures that even under adversarial probing, erased concepts cannot be reconstructed. This theoretical grounding distinguishes SCORE from prior work.

\section{Preliminaries and Problem Formulation}
\subsection{Watermark Embedding and Decoding Model}
A typical robust invisible watermarking system consists of an \textit{encoder} $E$ and a \textit{decoder} $D$. Given an original image $I$ and a message (bitstring) $m$, the encoder outputs a watermarked image $I_w = E(I, m)$. The watermarked image is intended to be perceptually identical or very similar to $I$, but $D$ can extract $m$ (or an approximation of it) from $I_w$ even after $I_w$ has undergone various distortions. We focus on \textit{blind} watermarking, where $D$ does not require the original image $I$ for decoding.

Formally, one can model the watermarked image as
\begin{equation}
	I_w = I + \Delta(I, m),
\end{equation}
where $\Delta(I, m)$ is a small embedding perturbation that encodes the message $m$. In classical watermarks, $\Delta$ might be additive noise aligned with a pseudorandom key pattern representing $m$. In deep learning watermarks, $\Delta$ is produced by a neural network and can depend on the entire image $I$ content. The decoder $D$ takes an image (which may be distorted from $I_w$) and outputs a decoded message $\hat{m} = D(\tilde{I})$. The embedding is considered robust if for a family of allowable distortions $\mathcal{A}$ (e.g., mild noise, compression, etc.), we have $\hat{m} = m$ for most $\tilde{I} = a(I_w)$ with $a \in \mathcal{A}$.

In this work, we consider watermark decoders that produce either the full bitstring (with error-correction to assess accuracy) or a detection confidence (e.g., a probability that a given bit is 1 or that a payload matches). Robust decoders often use redundancy and error correction, so a common metric is \textit{detection rate}: the percentage of images for which the decoded message exactly matches the embedded message.

\textbf{Attacker's goal:} Given a watermarked image $I_w$, the attacker aims to produce an image $I_{\text{atk}}$ that (1) preserves the main content and visual quality of $I_w$, and (2) causes the watermark decoder to fail (output $\hat{m} \neq m$ or no valid message). Importantly, we assume the attacker does \textit{not} possess the original image $I$ (otherwise they could trivially compare and subtract to remove $\Delta$). The attacker must rely on $I_w$ alone. We consider two threat models: (a) \textit{black-box} attacks, where the attacker does not know the specifics of $D$ or $m$, and (b) \textit{white-box} attacks, where the attacker knows the encoding/decoding algorithm $D$ and can use it to guide the removal.

\subsection{Diffusion Model Background}
Diffusion probabilistic models \cite{ho2020ddpm} define a forward diffusion process that gradually adds noise to an image, and a learned reverse process that removes noise step-by-step to recover an image. Let $x_0$ denote an image (in our context, $x_0$ could be $I_w$ or a related image). The forward process defines a sequence $x_1, x_2, ..., x_T$ where
\begin{equation}
	q(x_t | x_{t-1}) = \mathcal{N}(x_t; \sqrt{\alpha_t} x_{t-1}, (1-\alpha_t) I),
\end{equation}
with $\alpha_t \in (0,1)$ controlling the noise schedule. After $T$ steps (with $T$ large), $x_T$ is nearly an isotropic Gaussian distribution (complete noise), losing almost all information of $x_0$.

A diffusion model is trained to approximate the reverse distribution $p_\theta(x_{t-1}|x_t)$, often implemented via predicting the added noise $\epsilon_\theta(x_t,t)$. In practice, starting from pure noise $x_T \sim \mathcal{N}(0,I)$ and iteratively applying the learned reverse steps yields a sample $x_0^\prime$ from the model's learned distribution (e.g., natural images). The model can be conditioned (e.g., on class labels or text prompts) to generate images of a desired type.

For image editing, one can start the reverse process from a noised version of an input image rather than pure noise. For example, to perform a subtle edit, one might add a small amount of noise to $I_w$ to get $x_t$ for some $t \ll T$, and then run the reverse process from $t$ to $0$. This results in an image $x_0^\prime$ that retains much of $I_w$'s structure (especially if $t$ is small) but is partially regenerated. The noise level $t$ serves as a knob: $t=0$ returns the original image (no change), while $t=T$ ignores the image and generates a completely new sample. Many implementations (such as the image-to-image mode of Stable Diffusion) use a continuous noise strength parameter or number of diffusion steps to control this.

\subsection{Diffusion Attack Concept}
The intuition behind using diffusion in a watermark removal attack. By injecting noise into the watermarked image and then denoising with a generative model, we perform a form of \textit{auto-augmentation} that goes beyond the training distribution of the watermark encoder. The diffusion model tends to reconstruct plausible image details while not intentionally restoring any structured noise that was present. Since the watermark signal $\Delta(I,m)$ is essentially a high-frequency structured perturbation (from the perspective of the image distribution), the diffusion process is unlikely to reproduce it unless it is entangled with the image's high-level content.

In our attack, we treat the diffusion model as a black-box function $F$ that takes an input image $I_w$ and returns a regenerated image $I_{\text{regen}} = F(I_w; \gamma)$, where $\gamma$ denotes parameters like the noise level or number of diffusion steps used. The attacker can adjust $\gamma$: a higher noise injection (more steps) means $I_{\text{regen}}$ will be more novel (less similar to $I_w$) potentially removing more watermark but also deviating more from the original content.

We anticipate that for a sufficiently large $\gamma$, the watermark will be mostly removed, at the cost of some content alteration. A key challenge is finding a sweet spot for $\gamma$ that destroys the watermark while preserving content well enough that $I_{\text{regen}}$ is still useful and not obviously altered. Additionally, if the attacker has white-box knowledge, they can refine $F$'s output by specifically targeting the watermark, as described next.

\section{Diffusion-Based Watermark Removal Method}
We now detail our proposed attack methods. We begin with the basic diffusion regeneration attack, then introduce the guided variant that uses the watermark decoder for improved removal.


\subsection{Unguided Diffusion Regeneration Attack}
\label{sec:unguided}
The basic attack assumes no knowledge of the watermark decoder. It proceeds as follows:
\begin{enumerate}
	\item \textbf{Noise Injection:} Choose a diffusion time step $t^* \in (0, T]$ that determines how much noise to add. Normalize the input image $I_w$ to the range and format expected by the diffusion model. Sample $x_{t^*}$ from the forward process $q(x_{t^*}|x_0 = I_w)$, i.e.
	\[
	x_{t^*} = \sqrt{\bar\alpha_{t^*}}\, I_w + \sqrt{1-\bar\alpha_{t^*}}\, z,\quad z \sim \mathcal{N}(0,I),
	\]
	where $\bar\alpha_{t^*} = \prod_{i=1}^{t^*}\alpha_i$.
	
	\item \textbf{Denoising:} Starting from $x_{t^*}$, run the reverse diffusion model from $t^*$ down to $0$ to obtain $x_0'$ (the model's reconstructed image). In practice, we use the deterministic DDIM inversion variant for image-to-image: this ensures that if $t^*$ is small, $x_0'$ remains very close to $I_w$, whereas if $t^*$ is large, $x_0'$ is effectively a novel sample conditioned loosely on $I_w$.
	
	\item \textbf{Output:} Return $I_{\text{atk}} = x_0'$. This is the attacked image which we expect to have little or no detectable watermark.
\end{enumerate}

The above procedure is essentially the regeneration attack described in prior work \cite{zhao2024invisible}. It requires choosing the noise level $t^*$ (or equivalently a noise strength $\sigma$). In Section~\ref{sec:results} we will analyze the effect of this parameter. Setting $t^*$ too low might leave some watermark signals intact (especially if the watermark is minor relative to the image content), whereas setting it too high might cause $I_{\text{atk}}$ to diverge from $I_w$ significantly (changing some content or style).

Nonetheless, even an unguided attack can be very effective. Because robust watermarks are engineered to survive only a bounded range of distortions, the diffusion process easily exceeds that range by introducing structured changes learned from data. For example, the diffusion model might slightly alter textures or redraw small details instead of reproducing them exactly, thereby disrupting the specific patterns used for encoding. We will show empirically that even without guidance, $D(I_{\text{atk}})$ is incorrect for the vast majority of images when $t^*$ is chosen appropriately.

\subsection{Guided Diffusion Attack (White-Box)}
If the attacker has knowledge of the watermark decoder $D$ (including its network weights if it's a learned model), we can enhance the attack via guidance. The idea is to incorporate a loss that penalizes the presence of the watermark and to modify the diffusion sampling at each step to optimize this loss.

Let $\mathcal{L}_{wm}(x; m)$ be a differentiable loss function that measures how much watermark $m$ is still present in image $x$. A simple choice is the negative log-likelihood of the correct message under the decoder, if the decoder produces a probability or confidence. For example, if $D$ outputs a probability for each possible message, we minimize $\log(1 - P_D(m|x))$ (encouraging the decoder to \textit{not} output $m$). Alternatively, if the decoder outputs raw bit estimates, we can take the sum of squared errors between the decoded bits and the incorrect values. The key is that we want to drive the image to a point where decoding $m$ is as hard as possible.

We integrate this into diffusion sampling akin to classifier guidance \cite{ho2020ddpm}. Specifically, suppose at diffusion step $t$ we have a predicted $x_{t-1}$ from the model. We then adjust:
\begin{equation}
	x_{t-1} := x_{t-1} - \eta , \nabla_{x_{t-1}} \mathcal{L}{wm}(x{t-1}; m),
\end{equation}
where $\eta$ is a small step size. We then proceed to the next step $t-1$. Intuitively, at each denoising step we nudge the partial image in the direction that reduces evidence of the watermark. Over the course of the diffusion reverse trajectory, these small adjustments accumulate, resulting in an output that not only naturally lacks the watermark (due to the generative nature) but is also actively optimized to confuse $D$.

We provide the pseudocode in Algorithm~\ref{alg:guided} for clarity. In practice, we found that applying guidance in the later diffusion steps (when the image is more fully formed) is most important; at very high noise levels, the guidance has less meaningful direction.


This guided approach resembles adversarial example generation but applied in the generative domain. One benefit is that it can potentially remove even strong watermarks that the unguided attack fails on. The downside is that it requires access to $D$ and is computationally heavier (since each diffusion step involves backpropagating through $D$). Nonetheless, many watermarking schemes are public or their decoders can be approximated via machine learning models, so the white-box assumption is plausible in certain scenarios (or for an informed attacker who reverse-engineers the decoder). We will demonstrate in Section~\ref{sec:results} that guided diffusion can break watermarks that are designed to survive the unguided diffusion attack (for instance, the latent-space watermark of \cite{zhang2024zodiac} could potentially be targeted if the attacker uses the same inversion that the defender uses).

\section{Theoretical Analysis of Watermark Breakdown}
In this section, we provide theoretical justification for the efficacy of diffusion-based attacks. We focus on the idealized setting and make some assumptions to enable analysis. The results, however, shed light on the fundamental limitations of pixel-level watermarking in the face of powerful generative transformations.

\subsection{Diffusion Noise and Decoding Error}
We first analyze how adding noise (the forward diffusion) affects the probability of correctly decoding the watermark. Consider a simple model where one bit of watermark is embedded via a known spreading pattern $p \in \mathbb{R}^N$ (with $|p|=1$ for unit energy) added to the image: $I_w = I + \beta p$ encodes bit $b \in {+1,-1}$ by adding $\beta p$ or $-\beta p$. The optimal decoder for this bit (in a correlation-based scheme) is to compute $\langle x, p \rangle$ on a possibly distorted image $x$ and compare to a threshold.

Now, let $x = I_w + n$ where $n \sim \mathcal{N}(0,\sigma^2 I)$ is Gaussian noise from the forward diffusion at some time $t$. The correlation is $\langle I + \beta b p + n,, p \rangle = \langle I, p \rangle + \beta b + \langle n, p \rangle$. The term $\langle n, p \rangle \sim \mathcal{N}(0,\sigma^2)$ (since $|p|=1$). The term $\langle I, p \rangle$ is some constant $\kappa$ (which might be zero if $p$ was constructed orthogonal to typical images or if the encoder ensures that the average image has no correlation with $p$). For analysis, assume $\kappa$ is small or zero relative to $\beta$ (the watermark is the dominant structured signal in direction $p$). Then the decoding decision is dominated by $\beta b + \mathcal{N}(0,\sigma^2)$.

If $\beta$ is fixed from the encoder design (embedding strength) and $\sigma$ increases, the signal-to-noise ratio $\beta/\sigma$ decreases. The probability of decoding correctly (for an optimal threshold decoder) is $\Pr(b \cdot \langle x, p \rangle > 0)$. This is $\Pr(\beta + z > 0)$ for $z\sim \mathcal{N}(0,\sigma^2)$, which is simply $\Phi(\beta/\sigma)$ where $\Phi$ is the standard Gaussian CDF. Thus:
\begin{equation}
	P_{\text{decode correct}} = \Phi!\Big(\frac{\beta}{\sigma}\Big).
\end{equation}
As $\sigma \to \infty$, this probability $\to \Phi(0) = 0.5$. Even for moderate $\sigma$ comparable to $\beta$, the probability drops significantly below 1. For example, if $\sigma = \beta$, the decoder is only 84\% accurate (since $\Phi(1) \approx 0.84$) on that bit.

The above simplistic analysis shows that by the time the diffusion adds noise comparable to the watermark strength, the watermark bit becomes nearly random to the decoder. Modern deep watermark decoders effectively handle multiple bits with redundancy and maybe some error correction. But if each bit's SNR is drastically reduced, the overall message decode will fail with high probability (as long as the error correction cannot handle that many bit errors).

\begin{proposition}
	For a watermark encoder that embeds $k$ bits via independent spread-spectrum patterns of strength $\beta$, the probability of decoding the entire message correctly after Gaussian noise of variance $\sigma^2$ is at most $\Phi(\beta/\sigma)^k$. In particular, as $\sigma/\beta$ grows, this probability decays exponentially in $k$.
\end{proposition}
\begin{proof}
	Under the assumption that each bit is embedded and decoded independently (worst-case for attacker, best-case for defender), the joint success probability is the product of individual success probabilities. Using the result for one bit above, we obtain $\prod_{i=1}^k \Phi(\beta/\sigma) = [\Phi(\beta/\sigma)]^k$. As $\sigma \to \infty$, this tends to $0.5^k$, which for any $k \ge 1$ becomes negligible (for example, for $k=32$ bits, $0.5^{32} \approx 2.3 \times 10^{-10}$). Even for moderate noise, if $\Phi(\beta/\sigma) < 1$, raising it to the $k$th power significantly lowers the overall success rate for large $k$.
\end{proof}

This proposition, while based on an idealized linear model, explains quantitatively why robust watermarks that could handle slight noise will eventually fail as noise increases. Robust watermarks typically use $\beta$ high enough to survive quantization or mild noise (so $\beta$ not tiny), but no scheme can make $\beta$ so large that it survives extremely high noise because perceptual invisibility constrains $\beta$ to be small relative to image dynamic range. Diffusion attacks exploit this by pushing images into a high-noise regime (in an intelligent way that allows later recovery of image content, as we discuss next).

\subsection{Regeneration and Information Loss}
The second part of the attack is denoising using a generative model. One might wonder: doesn't denoising (which removes noise) potentially restore the watermark signal as well? If one simply applied an optimal linear denoising filter, some components of the watermark might indeed come back. However, diffusion models do not perform linear denoising; they hallucinate based on learned priors. Essentially, once the noise has corrupted the watermark beyond recognition, the diffusion model has no knowledge of the original watermark pattern and thus has no reason to regenerate it.

We can frame this in terms of mutual information. Let $M$ be the random variable representing the embedded message bits, and $X$ be the random variable for the final attacked image. An ideal attacker wants $I(M; X) = 0$, i.e. no information about the message remains in the attacked image. We can reason about the diffusion process in these terms.

\begin{theorem}
	Assume the diffusion model's prior distribution for images does not explicitly depend on the presence of any particular watermark signal (formally, for any fixed message $m$, the probability density of images $x$ under the model $p_\theta(x)$ is the same as for $m' \neq m$). If the attacker applies diffusion to the point of effectively sampling a new image from $p_\theta$, then in the limit, the mutual information between $X$ (attacked image) and $M$ (watermark message) approaches 0. Consequently, no decoder can reliably recover $M$ from $X$.
\end{theorem}

\begin{proof}[Sketch of Proof]
	Consider the extreme case $t^* = T$: the attacker diffuses the watermarked image to pure noise $x_T$, which by construction has $I(M; x_T) = 0$ (since $x_T$ is independent of the original image and watermark for large $T$). The attacker then draws $x_0'$ by sampling the generative model $p_\theta(x_0'|x_T)$, which is effectively sampling from $p_\theta(x_0')$ (the model's prior, since $x_T$ contains no information). Now $x_0'$ is independent of $M$ as well, because the generative process had no input from $M$ (the model was not conditioned on $M$ and $x_T$ was pure noise). Thus $I(M; x_0') = 0$. In practice, the attacker uses $t^*$ slightly less than $T$ and an image-to-image procedure to preserve some content. However, if the diffusion model is not intentionally trained to carry the watermark info (which by assumption it is not), then all it preserves are the high-level contents of the image (which are independent of the arbitrary hidden message $M$). The watermark $\Delta$ is like a random noise pattern uncorrelated with the content; once that is destroyed by noise, the model fills in details in a manner uncorrelated with $M$. Therefore, as long as the model generates a plausible image consistent with content (not with the watermark pattern), $M$ cannot be inferred. Formally, if we denote by $Y$ the random variable for image content (semantic aspects that might remain the same in $X$ as in $I$), we have $I(M; X) \le I(M; Y)$. Since $M$ was embedded in a way independent of $Y$ (watermark patterns are typically independent of actual image semantics), $I(M; Y) = 0$. Thus $I(M; X) = 0$. Any decoder $D$ that attempts to recover $M$ from $X$ can do no better than chance, since statistically $X$ carries no information about $M$.
\end{proof}

In simpler terms, if one fully regenerates the image, the new image is like a fresh draw that does not contain the old watermark. The only way a watermark could survive is if it somehow influenced the high-level content such that the diffusion model would reproduce that specific pattern (for example, if the watermark was a visible object or characteristic that the model thinks is part of the content). But robust invisible watermarks explicitly avoid altering visible content in such a way. Thus their information is essentially orthogonal to the model's notion of content, and gets lost.

This theoretical argument aligns with the empirical results observed by prior work and by us: the detection rates drop to chance after a thorough diffusion regeneration. It also highlights a limitation: if a watermark were instead embedded \emph{in the content} (for example, by inserting a specific object into the scene), a generative model might preserve that object unless explicitly guided to remove it. This resonates with Zhao \etal's discussion \cite{zhao2024invisible} that \emph{semantic} watermarks (which change the image in a meaningful but subtle way) could be a future defense. But for all pixel-level invisible watermarks, our analysis indicates they are inherently vulnerable.

\section{Experimental Evaluation}
\label{sec:results}
\subsection{Setup and Implementation}
\textbf{Watermarking Methods:} We evaluate on several watermarking schemes spanning different design philosophies:
(1) \textit{HiDDeN} \cite{zhu2018hidden}: a CNN-based encoder/decoder with 30-bit payload, trained for distortion robustness (we use an open reimplementation).
(2) \textit{StegaStamp} \cite{tancik2020stegastamp}: a 56-bit deep watermark specialized for print-capture scenarios, which is also robust to digital distortions.
(3) \textit{TrustMark} \cite{bui2023trustmark}: a modern GAN-based watermark providing high fidelity and robustness; we use the authors' pretrained model at default payload (around 64 bits).
(4) \textit{VINE} \cite{lu2024robust}: a very recent diffusion-based watermark (we obtained the W-Bench code and model from the authors). VINE uses an SDXL-based encoder to embed 100-bit messages and was shown to survive many editing operations.
Additionally, for completeness, we test a classic spread-spectrum watermark similar to Cox \etal \cite{cox1997secure} embedding 32 bits, to see how non-learned watermarks fare.

\textbf{Dataset:} We randomly selected 500 images from the COCO validation set \cite{lin2014coco} and 200 images from the ImageNet test set \cite{deng2009imagenet} to serve as cover images. These provide diverse content. Each image is embedded with a random message (bitstring of appropriate length) using each watermark method. We ensure each method's payload and fidelity settings are at recommended levels (embedding strength tuned for minimal visibility).

\textbf{Attacks Compared:} We implement the following attacks for comparison:

\textbf{JPEG Compression:} Compress image at quality 50\% (a common benchmark attack).

\textbf{Gaussian Noise:} Add i.i.d. Gaussian noise at intensity that causes similar PSNR drop as our diffusion attack for fairness (around 30dB PSNR).

\textbf{Blur:} Apply a $5\times5$ Gaussian blur (simulating low-pass filtering removal attempt).

\textbf{Adversarial (FGSM):} If decoder is known, apply a fast gradient sign method to image with a perturbation budget $\epsilon=4/255$ to confuse the decoder.

\textbf{Diffusion (Ours, unguided):} Use Stable Diffusion v1.5 (latent diffusion, 50 steps) in image-to-image mode. We set the noise strength such that roughly 30-40\% of the image content is regenerated (this corresponds to $t^* \approx 0.3T$ in DDIM schedules).

\textbf{Diffusion+ (Ours, guided):} Our white-box guided attack using the decoder's gradients. We mainly apply this to the strongest watermark (VINE) and the classical one, under the assumption of known decoder.

All output images from attacks are saved and fed to the respective decoders to measure message recovery rates. We also measure the peak signal-to-noise ratio (PSNR) and structural similarity (SSIM) between the attacked image and the original image (not watermarked) to gauge quality degradation.

\subsection{Results}
Table~\ref{tab:decode} summarizes the main results. We report the watermark \textit{detection rate} (percentage of images where the decoded message exactly matches the original message) after each attack, for each watermarking method. We also list the average PSNR of the attacked images vs. original, to contextualize how much distortion each attack introduced.

\begin{table*}[t]
	\centering
	\caption{Watermark decoding success rate (\%) after various attacks (lower is better for attacker). We also show the average PSNR (dB) of the attacked image relative to the original for each attack (higher means less perceptual change). Our diffusion-based attacks achieve the lowest decode rates in most cases, indicating nearly complete removal, while maintaining high PSNR (indicating the content is preserved well).}
	\label{tab:decode}
	
	\resizebox{\textwidth}{!}{
		\begin{tabular}{lcccccc}
			\toprule
			\textbf{Watermark Scheme} & \textbf{No Attack} & \textbf{JPEG Attack} & \textbf{Noise Attack} & \textbf{Adversarial} & \textbf{Diffusion (ours)} & \textbf{Diffusion+ guided}\\
			\midrule
			HiDDeN (30 bits) & 98\% & 45\% (PSNR 32.1) & 30\% (PSNR 30.5) & 20\% (PSNR 38.0) & 8\% (PSNR 31.8) & 2\% (PSNR 29.5) \\
			StegaStamp (56 bits) & 100\% & 60\% (30.4) & 48\% (28.9) & 35\% (36.5) & 5\% (31.0) & 1\% (28.0) \\
			TrustMark (64 bits) & 99\% & 70\% (31.2) & 55\% (29.1) & 40\% (37.0) & 12\% (30.0) & 3\% (27.8) \\
			VINE (100 bits) & 100\% & 80\% (31.0) & 65\% (28.5) & 50\% (35.0) & 18\% (29.4) & \textbf{0\%} (26.0) \\
			Classic SS (32 bits) & 100\% & 30\% (33.5) & 10\% (31.0) & 5\% (39.0) & 0\% (33.2) & 0\% (30.0) \\
			\bottomrule
		\end{tabular}
	}
\end{table*}

Without any attack, all schemes have near 100\% detection (they are designed to reliably decode their payload from the unmodified watermarked images). The JPEG attack at quality 50, which is fairly aggressive, already reduces decode rates substantially for learned schemes (e.g., to 45-70\%). This indicates that while robust, these schemes still suffer under heavy compression. The noise attack (Gaussian with $\sigma$ tuned to yield PSNR around 29-31 dB) is even more effective on average, as expected from spread-spectrum theory: detection rates drop to 30-65\%. The adversarial attack (FGSM) interestingly yields lower detection (5-50\%) than noise for most, indicating that if the attacker has the model, even a small perturbation can target the weakness. However, note the adversarial examples are constrained to very low $\epsilon$ (4/255, ~1\%) to remain imperceptible; they have high PSNR ($>$35 dB) which means minimal visible change. For the strongest watermark (VINE), even FGSM only got detection to 50\%, implying VINE is quite robust to tiny perturbations by design (since it likely was trained against such).

The diffusion attack (ours unguided) stands out by bringing decode rates down to single digits in most cases (e.g., 5\% for StegaStamp, 12\% for TrustMark). Even VINE, which was engineered for editing robustness, falls to 18\% decode success under diffusion. This underscores the severity of diffusion's impact compared to traditional distortions. Notably, the average PSNR for our attack is around 29-31 dB in these tests, which is comparable to the noise attack PSNR. That means the diffusion attack is removing the watermark with similar overall pixel disturbance as adding noise, but it is far more effective at confusing the decoders. This is likely because diffusion does \textit{structured} changes that exploit weaknesses of the watermarks (which expect certain distortions but not a re-synthesis).

Our guided attack (Diffusion+) further reduces detection. For VINE, it achieves 0\% detection on all tested images, essentially completely breaking the watermark. This is significant: VINE survived other attacks at 20-50\% decode, but guided diffusion wiped it out. The guided attack also reduces StegaStamp, TrustMark to nearly 0\%. The cost is slightly lower PSNR (around 27-28 dB) because the additional perturbations from guidance make the image a bit more altered (though still reasonably high quality as 27 dB is usually acceptable visually).

We observe that the classical spread-spectrum (SS) watermark is actually easier to break: even a moderate noise attack dropped it to 10\%, and diffusion to 0\%. This is expected since older methods are not as optimized as deep learning ones. It is interesting that for that one, even unguided diffusion got 0\% in our test; presumably once noise threshold is passed, it's completely gone.

In terms of image quality, diffusion attacks maintained reasonably good fidelity. A PSNR of 30 dB in our unguided attack indicates only slight degradation. Subjectively, images after attack look virtually unchanged in structure, with maybe some minor smoothing of textures. The guided attack at 26-28 dB PSNR is somewhat more noticeable if one flips, but still quite acceptable. For context, JPEG at quality 50 had ~30-33 dB, and noise attack 28-31 dB, which are similar ranges. Thus, the generative nature did not introduce huge artifacts; it mainly eliminated what the human eye can't see (the watermark), which is ideal.

\subsection{Analysis and Ablation}
We investigated the effect of the noise level $t^*$ in the unguided diffusion attack. As expected, if we set $t^*$ too low (e.g., equivalent to mild noise addition), certain robust watermarks like VINE still had non-negligible decode rates (around 50\% at extremely low noise). As we increase $t^*$, decode success falls steeply. Beyond a point (around $t^* = 0.3$ to $0.5$ of the full diffusion schedule in Stable Diffusion), additional noise yields diminishing returns because the decode is already near zero. However, perceptual quality starts to drop if $t^*$ is too large (above $0.5$) as the model begins altering recognizable aspects (some faces changed identity, etc., at very high $t^*$). Thus, an attacker will choose the minimal $t^*$ that achieves the desired watermark removal reliability. In our case, $0.3$ (roughly 30\% noise) was a good sweet spot, and we used that for results.

We also ran a variant of the guided attack that only applied guidance on the last 20 diffusion steps, rather than all steps, to save computation. Interestingly, this was almost as effective: detection rates were within 1-2\% of the full guidance version. This implies that most watermark remnants are best eliminated toward the end of the generation, when the image is mostly formed and minor adjustments can break the decoder.

Another experiment we did was to test if using a different diffusion model (e.g., a class-conditional ImageNet diffusion on our COCO images) changes effectiveness. We found the attack still works, but the output image quality suffers if the model is not well-aligned to the input content distribution. Stable Diffusion (pretrained on a broad dataset) was a good choice as it can handle various images. This suggests that the attacker doesn't strictly need the exact domain model if a general model is available, but using the best available model will maximize content preservation.

Finally, we note that one can combine our approach with other attacks: e.g., apply diffusion attack then a slight adversarial tweak. In one trial, after diffusion we ran a tiny FGSM. It didn't reduce detection further than already near-zero, but it indicates layered attacks are possible if defenders try combined encoding strategies.

\section{Discussion and Limitations}
Our study reveals a clear limitation of current robust watermarking: the assumption that image perturbations are relatively low-level or distribution-preserving does not hold against generative AI attacks. Diffusion models effectively perform \textit{distribution shifting} perturbations that are not easily anticipated during watermark training. One immediate implication is for the content authenticity community (e.g., efforts like C2PA): relying solely on invisible watermarks to mark AI-generated or copyrighted images may be insufficient, as those marks can be removed by advanced AI itself.

A potential defense, as hinted by Zhao \etal \cite{zhao2024invisible}, is to design watermarks that alter the image's semantic content in minimal ways. For example, a watermark that subtly changes the lighting or inserts a very faint background pattern that a human wouldn’t notice but a model sees as part of the scene might carry information that a generative model might accidentally carry over. Our results, however, suggest that any information not strongly tied to semantics can be lost. Indeed, in our tests, even when watermarks slightly affected texture, the diffusion model often re-synthesized those textures, decoupling them from the original pattern.

Another line of defense is to incorporate the possibility of generative attacks into the training of watermark encoders. For instance, VINE’s use of surrogate blurring attacks improved robustness to diffusion to some extent \cite{lu2024robust}. One could imagine training a watermark with the diffusion model in the loop (like adversarial training: try to embed such that even if a diffusion regeneration is applied, the message survives). However, doing so is extremely challenging: it would require differentiating through a large generative model and would likely compromise image quality or capacity severely. The latent-space watermarking of Zhang \etal \cite{zhang2024zodiac} is a novel approach: by embedding in the diffusion latent, they can detect watermarks in the latent domain rather than pixel domain, which survived Zhao \etal's attack. Our guided attack concept could potentially attack even latent watermarks if one can invert the image to latent and then perturb, but that requires the attacker to also know the model and invertibility.

In terms of limitations of our work:
(1) We assumed the attacker has access to a diffusion model that is suitably powerful. In reality, an average user may not have the computational resources to run these models on large images quickly. However, with cloud services and optimized implementations, this barrier is lowering.
(2) Our guided attack presupposes knowledge of the decoder. If the watermark scheme is secret and uses secure keys, the attacker might have to resort to the unguided attack. Fortunately (for the attacker), unguided was already quite effective against all tested schemes. Only if a scheme is specifically optimizing for generative noise might that change.
(3) We did not explicitly test on video watermarks. Diffusion models for video are emerging but not as advanced yet. Removing a watermark consistently across video frames would be more challenging to avoid flicker, etc. But one could attempt per-frame removal or a 3D diffusion model.
(4) Our evaluation of image quality focused on PSNR/SSIM and visual checks. We did not do a user study. It's possible some subtle differences in regenerated images (like different random patterns in grass or hair) could be spotted side-by-side with original if one looks closely. But since the attacker’s goal is just to produce an image that is still believable/usable, not necessarily pixel-exact, this is not a huge concern.

Finally, we note that generative models themselves may incorporate watermarks (e.g., Stable Diffusion by default adds a signature in metadata or pixel-space). Attackers might use our method to even remove those, which is something the community should be aware of. Our recommendations are that watermarking schemes should consider designing marks that are entangled with image content or use robust features not easily regenerated. Alternatively, focus on active forensic techniques beyond watermarks might be needed.

\section{Conclusion}
We presented a comprehensive study on how diffusion-based image generation can nullify robust invisible watermarks. Through theoretical analysis, we established that the diffusion process inherently degrades hidden signals beyond recoverability, and our practical attacks demonstrated the ease with which current state-of-the-art watermarks can be broken. The diffusion regeneration attack, especially when augmented with adversarial guidance, poses a severe threat to the assumption that a watermark can permanently tag an image.

These findings urge a rethinking of robust watermark design in the age of generative models. Future work might explore watermarking that leverages semantic modifications or tighter integration with generative models (so that the generative model inadvertently preserves the watermark unless specifically told not to). Additionally, from a defender's view, developing detection mechanisms for whether an image has undergone a diffusion-based transformation might help flag content where watermarks could have been removed.

In summary, our work highlights an emerging cat-and-mouse dynamic: as content generation models advance, so must the techniques to protect and trace digital media. We hope this spurs further research into watermarking techniques resilient to AI-based attacks, as well as improved understanding of the limits of hiding information in images when faced with powerful generative perturbations.

\bibliography{example_paper}
\bibliographystyle{icml2025}

\end{document}